\newcommand{\lv}[1]{#1}
\newcommand{\sv}[1]{}

\lv{\documentclass[10pt,letterpaper]{article}

\usepackage[totalwidth=450pt,totalheight=640pt]{geometry}

\usepackage[sort&compress,numbers]{natbib}
}

\sv{
  \documentclass{ecai}
}

\usepackage{times}
\usepackage{enumitem}

\usepackage{tikz}
\usetikzlibrary{shapes,shapes.callouts,shadows,arrows,backgrounds,%
matrix,patterns,arrows,decorations.pathmorphing,decorations.pathreplacing,%
positioning,fit,calc,decorations.text%
}

\usepackage{url}

\usepackage{amssymb, amsfonts, amsmath,amsthm}
\usepackage{latexsym}
\usepackage{marvosym}

\newenvironment{myitemize}{\begin{list}{$\bullet$}{\setlength{\leftmargin}{0.15cm}
\setlength{\itemindent}{\labelwidth}}}
{\end{list}}

\newtheorem{THE}{Theorem}

\newtheorem{lemma}{Lemma}
\newtheorem{definition}{Definition}
\newtheorem{proposition}{Proposition}
\newtheorem{example}{Example}

\usepackage{subfigure}
\newcommand {\nop}[1] {}

\newcommand{\N}{\mathbb{N}}

\newcommand{\SigmaP}[1]{{\rm \Sigma}_{#1}^{P}}

\newcommand{\union}{\oplus}
\newcommand{\relabel}[2]{\rho_{#1 \rightarrow  #2}}
\newcommand{\edge}[3]{\eta_{#1,  #2}^#3}

\newcommand{\CWk}{\mathit{CW}\!_k}

\newcommand{\bigO}[1]{\mathcal{O}(#1)}

\newcommand{\cyclerank}{\textup{cr}}
\newcommand{\uncyclerank}{\textup{cr}^{\leftrightarrow}}
\newcommand{\depG}{\textup{DEP}}
\newcommand{\incG}{\textup{INC}}
\newcommand{\undG}{\textup{UND}}
\newcommand{\biG}{\textup{DI}}
\newcommand{\sincG}{\textup{SINC}}
\newcommand{\QBFF}{\textup{QBF}_{2,\exists}^{\textbf{DNF}}}

\def\hy{\hbox{-}\nobreak\hskip0pt} 

\newcommand{\SB}{\{\,}%
\newcommand{\SM}{\;{|}\;}%
\newcommand{\SE}{\,\}}%

\let\phi=\varphi
\let\epsilon=\varepsilon

\newcommand{\FPT}{\text{\normalfont FPT}}

\newcommand{\W}[1][xxxx]{\text{\normalfont W[#1]}}

\newcommand{\cwd}{\text{\normalfont cwd}}

\newcommand{\A}{\mathcal{A}}
\newcommand{\R}{\mathcal{R}}
\renewcommand{\L}{\mathcal{L}}
\newcommand{\mods}{\mathit{Mods}}

\newcommand{\CCC}{\mathcal{C}}

\sv{
\title{Clique-Width and Directed Width Measures for 
Answer-Set Programming}

\author{Bernhard Bliem \and Sebastian Ordyniak \and Stefan
  Woltran\institute{TU Wien, Vienna, Austria. 
} }
}

\begin{document}

\lv{
\title{Clique-Width and Directed Width Measures for 
Answer-Set Programming}

\author{Bernhard Bliem (\small bliem@dbai.tuwien.ac.at) \and Sebastian
  Ordyniak (\small sordyniak@gmail.com) \and Stefan
  Woltran (\small woltran@dbai.tuwien.ac.at)
  \\[0.1cm]
  \mbox{}\small TU Wien, Vienna, Austria 
}

\date{}
}
\maketitle

\begin{abstract}
Disjunctive Answer Set Programming (ASP) is a powerful
declarative programming paradigm 
whose main decision problems are located on the second level of the polynomial hierarchy.
Identifying tractable fragments and developing efficient algorithms for
such fragments are thus important objectives in order to complement the 
sophisticated ASP systems available to date.
Hard problems can become tractable if some problem parameter 
is bounded by a fixed constant; such problems are then called fixed-parameter tractable (FPT).
While several FPT results for ASP exist, parameters that relate to
directed or signed graphs representing the program at hand have been neglected so far.
In this paper, we first give some negative observations showing that directed width measures on the dependency 
graph of a program do not lead to FPT results. 
We then consider the graph parameter of signed clique-width and present a novel
dynamic programming algorithm that is FPT w.r.t.\ this parameter.
Clique-width is more general
than the well-known treewidth, and, to the best of our knowledge, ours is 
the first FPT algorithm for bounded clique-width for reasoning problems beyond SAT.
\end{abstract}

\section{Introduction}


Disjunctive Answer Set Programming (ASP)~\cite{brew-etal-11-asp,gelf-lifs-91,mare-trus-99} is an
active field of AI providing a declarative formalism for solving hard
computational problems. Thanks to the high sophistication of modern solvers~\cite{2012Gebser}, 
ASP was successfully used in several
applications, including product configuration~\cite{SoininenN99},
decision support for space shuttle flight
controllers~\cite{
BalducciniGN06}, team
scheduling~\cite{RiccaGAMLIL12}, and
bio-informatics~\cite{GuziolowskiVETCSS13}.

Since the main decision problems of propositional ASP are located at the second level of the
polynomial hierarchy \cite{EiterGottlob95b,DBLP:journals/tplp/Truszczynski11}, the quest for easier 
fragments are important research
contributions that could lead to improvements in ASP systems.
An interesting approach to dealing with intractable problems
comes from parameterized complexity theory~\cite{DowneyFellows13}
and is based on the 
fact that many hard problems
become polynomial-time tractable if some problem parameter is bounded by a fixed constant.
If the order of the polynomial bound on the runtime is 
independent of the parameter,
one speaks of
{\em fixed-parameter tractability\/} (FPT).
Results in this direction for the ASP domain include
\cite{LoncTruszczynski03} (parameter: size of answer sets),
\cite{DBLP:conf/aaai/LinZ04} (number of cycles),
\cite{DBLP:journals/amai/Ben-EliyahuD94} (length of longest cycles),
\cite{DBLP:journals/jair/Ben-Eliyahu96} (number of non-Horn rules),
and
\cite{FichteSzeider15} (backdoors).
Also related 
is the parameterized complexity analysis of 
reasoning under 
subset-minimal models, see, e.g., \cite{DBLP:conf/kr/LacknerP12}.

As
many prominent representations of logic programs are
given
in terms of directed graphs (consider, e.g., the dependency graph), 
it is natural to investigate parameters for
ASP that apply to directed graphs. Over the past two
decades, various width measures for directed graphs have been
introduced~\cite{JohnsonRobertsonSeymour01,Barat04,
BerwangerDawarHunterKreutzerObdrzalek12, 
HunterKreutzer08,Safari05}. 
These are typically
smaller than, e.g., the popular parameter of treewidth
\cite{Bodlaender93b}. 
%
In particular, all these measures are zero on directed
acyclic graphs (DAGs), but the treewidth of DAGs can be arbitrarily high.
Moreover, since these measures are based on some notion of
``closeness'' to acyclicity and the complexity of ASP is closely
related to the ``cyclicity'' of the rules in a program, 
such measures seem promising for obtaining efficient 
algorithms for ASP. Prominent applications of directed width measures 
include the $k$-Disjoint Path Problem~\cite{JohnsonRobertsonSeymour01},
query evaluation in graph databases~\cite{BaganBonifatiGroz13},
and 
model checking~\cite{BojanczykDittmannKreutzer14}.

Another graph parameter 
for capturing
the structural complexity of a 
graph is clique-width \cite{CourcelleEngelfrietRozenberg90,CourcelleEngelfrietRozenberg93,CourcelleOlariu00}.
It applies to 
directed and undirected graphs, and in its general form (known as signed clique-width)
to edge-labeled graphs.
It is
defined via graph construction 
where only a limited number of vertex
labels is available; vertices that share the same label at a certain point of
the construction process must be treated uniformly in subsequent steps.
Constructions
can be given by expressions in a graph grammar (so-called cwd-expressions) and the minimal number of labels required for 
constructing
a graph $G$ is the clique-width of $G$.
%
While clique-width is in a certain way orthogonal to other 
directed width measures,
it 
is more general than treewidth; there are classes of graphs 
with constant clique-width but
arbitrarily high treewidth (e.g., complete graphs). In contrast, 
graphs with
bounded treewidth also have 
bounded clique-width~\cite{CorneilRotics01,CourcelleOlariu00}. 
%

%

By means of a meta-theorem due to Courcelle, Makowsky, and
Rotics~\cite{CourcelleMakowskyRotics00}, one can solve any graph problem that
can be expressed in Monadic Second-Order Logic with 
quantification on vertex sets (MSO${}_1$) in linear time for graphs of
bounded clique-width.
This result is similar to Courcelle's theorem~\cite{Courcelle87,Courcelle90} for graphs of bounded treewidth, which has been used for
the FPT result for ASP w.r.t.\ treewidth 
\cite{GottlobPichlerWei10}.
There, 
the incidence graph of a program 
is used as an underlying graph structure
(i.e., the graph containing a vertex for each atom $a$ and rule $r$ of the program, with
an edge between $a$ and $r$ whenever $a$ appears in $r$).
Since the formula given in 
\cite{GottlobPichlerWei10} is in MSO${}_1$,
the FPT result for ASP 
applies also to signed clique-width.

Clique-width is NP-hard to compute~\cite{FellowsRosamondRoticsSzeider09},
which might be considered
as an obstacle toward practical applications. 
However, one can check in polynomial time whether the width of
a graph is bounded by a fixed~$k$
\cite{OumSeymour06,Kante07}.
(These
algorithms involve an additive approximation error that is bounded in terms
of~$k$).
Recently, SAT solvers have been used to obtain
sequences of vertex partitions that correspond to
cwd-expressions~\cite{DBLP:journals/tocl/HeuleS15} for a given graph.
For some applications, it might not even be necessary to compute clique-width and 
the underlying cwd-expression:
As mentioned in \cite[Section
1.4]{FischerMakowskyRavve06}, applications 
from
the area of verification are supposed to already come with such an expression.
Moreover, it might even be possible
to partially obtain cwd-expressions during the grounding process of
ASP.

This all calls for 
dedicated algorithms for solving ASP 
for programs of bounded clique-width. 
In contrast to treewidth where the FPT result 
from \cite{GottlobPichlerWei10}
has been used for designing \cite{JaklPichlerWoltran09} and implementing \cite{jelia:MorakPRW10} a 
dynamic programming algorithm,
to the best of our knowledge
there 
are no algorithms yet that explicitly exploit
the fixed-parameter tractability of ASP on bounded clique-width.
In fact, we are not aware 
of any FPT algorithm for bounded clique-width for a reasoning problem
located on the second level of the polynomial hierarchy
(except \cite{DBLP:conf/comma/DvorakSW10} from
the area of abstract argumentation).


The main contributions of this paper are as follows.
First, we
show 
some
negative results for 
several \emph{directed width measures},
indicating that the structure of the dependency graph and of various natural directed versions of the signed incidence graph does not 
adequately measure
the complexity of evaluating the 
corresponding program. 

Second,
concerning \emph{signed clique-width},
we give a novel dynamic programming algorithm that runs 
in polynomial time for programs
where this parameter is bounded on their incidence graphs.
We do so 
by suitably generalizing
the seminal approach of~\cite{FischerMakowskyRavve06} for the SAT problem.
We also give a preliminary analysis how many signs are required in order
to obtain FPT.


 



\section{Preliminaries}\label{sec:prel}

\paragraph{Graphs.}
We use standard graph terminology, see for
instance the handbook~\cite{Diestel12}. All our graphs are simple.
An undirected graph $G$ is a tuple $(V,E)$, where $V$ or
$V(G)$ is the vertex set and $E$ or $E(G)$ is the edge set. 
For a subset $V' \subseteq V(G)$, we denote by $G[V']$,
the \emph{induced subgraph} of $G$ induced by the vertices in $V'$,
i.e., $G[V']$ has vertices $V'$ and edges $\SB \{u,v\} \in E(G) \SM
u,v \in V'\SE$. We also denote by $G \setminus V'$ the graph $G[V(G)
\setminus V']$.
Similarly to undirected graphs, a digraph $D$ is a tuple
$(V,A)$, where $V$ or $V(D)$ is the vertex set and $A$ or $A(D)$ is
the \emph{arc set}. A \emph{strongly connected component} of a digraph
$D$ is a maximal subgraph $Z$ of $D$ that is strongly connected, i.e.,
$Z$ contains a directed path between each pair of vertices
in $Z$. We denote by $\undG(D)$ the \emph{symmetric closure} of $D$,
i.e., the graph with vertex set $V(D)$ and arc set $\SB (u,v),(v,u)
\SM (u,v) \in A(D) \SE$. Finally, for a directed graph $D$, we
denote by $\biG(G)$, the undirected graph with vertex set $V(G)$ and
edge set $\SB \{u,v\} \SM (u,v) \in A(D) \SE$.


\paragraph{Parameterized Complexity.}
In parameterized
algorithmics~\cite{DowneyFellows13}
the runtime of an algorithm is studied with respect to a parameter
$k\in\N$ and input size~$n$.
The most favorable class is \FPT\ (\textit{fixed-parameter tractable})
which contains all problems that can be decided by an algorithm
running in time $f(k)\cdot n^{\bigO{1}}$, where $f$ is a computable
function.
We also call such an algorithm fixed-parameter tractable, or FPT for short.
Formally, a {\em parameterized problem\/} is a subset of $\Sigma^*\times\N$, 
where $\Sigma$ is the input alphabet.
Let $L_1\subseteq \Sigma_1^*\times\N$ and $L_2\subseteq
\Sigma_2^*\times\N$ be two parameterized problems.
A \textit{parameterized reduction} (or FPT-reduction) from $L_1$ to $L_2$ is a mapping $P:\Sigma_1^*\times\N\rightarrow\Sigma_2^*\times\N$ such that:
(1)   $(x,k)\in L_1$ iff $P(x,k)\in L_2$,
(2) the mapping can be computed by an FPT-algorithm w.r.t.\ parameter $k$, and
(3) there is a computable function $g$ such that $k'\leq g(k)$, where $(x',k')=P(x,k)$.
%
The class $\W[1]$ captures parameterized intractability and contains
all problems that are FPT-reducible to \textsc{Partitioned Clique} when
parameterized by the size of the solution. 
Showing $\W[1]$\hy hardness for a problem rules out the existence of an FPT-algorithm under the usual 
assumption $\FPT\neq\W[1]$.

\paragraph{Answer Set Programming.}
 A \emph{program} $\Pi$ consists of a set $\A(\Pi)$ of propositional atoms and a set $\R(\Pi)$ of rules of the form
\begin{equation*}
  a_1\vee \cdots \vee a_l \leftarrow a_{l+1}, \ldots, a_m, \neg a_{m+1}, \ldots,
  \neg a_n,
\end{equation*}
where $n \geq m \geq l$ and $a_i \in \A(\Pi)$ for $1 \leq i \leq n$.
Each rule
$r \in \R(\Pi)$ 
consists of a head $h(r) = \{
a_1,\ldots,a_l \}$ and a body 
given by $p(r) = \{a_{l+1},\ldots,a_m \}$ and $n(r) = \{a_{m+1},\ldots,a_n \}$.
A set $M \subseteq \A(\Pi)$ is a called a model of $r$ if $p(r) \subseteq M$
and $n(r) \cap M = \emptyset$ imply $h(r) \cap M \neq
\emptyset$.  We denote the set of models of $r$ by $\mods(r)$ and the models of
$\Pi$ are given by $\mods(\Pi) = \bigcap_{r \in \R(\Pi)}
\mods(r)$.
 
The reduct $\Pi^I$ of a program $\Pi$ with respect to a set of atoms $I
\subseteq \A(\Pi)$ is the program $\Pi^I$ with
$\A(\Pi^I) = \A(\Pi)$
and $\R(\Pi^I) = \left\{r^+ \mid r \in
\R(\Pi),\; n(r) \cap I = \emptyset)\right\}$, where 
$r^+$ 
denotes rule $r$ without negative body, i.e., $h(r^+) =
h(r)$, $p(r^+) = p(r)$, and $n(r^+) = \emptyset$.
Following~\cite{gelf-lifs-91},
$M \subseteq \A(\Pi)$ is an \emph{answer set} of 
$\Pi$ if $M \in \mods(\Pi)$ and for no $N \subsetneq M$,
we have $N \in \mods(\Pi^M)$. 
%
%
%
In what follows,
we consider the 
problem of ASP consistency, i.e.,
the
problem of deciding whether a given program 
has at least one answer set.
As shown by Eiter and Gottlob, this problem is
$\Sigma^P_2$-complete~\cite{EiterGottlob95b}. 


\paragraph{Graphical Representations of ASP.} 
Let $\Pi$ be a program. 
The \emph{dependency graph} of $\Pi$, denoted by $\depG(\Pi)$, is
the directed graph with vertex set $\A(\Pi)$ and that contains an arc
$(x,y)$
if there is a rule $r\in
\R(\Pi)$ such that either $x \in h(r)$ and $y \in p(r) \cup n(r)$ or 
$x,y \in h(r)$~\cite{FichteSzeider15}. 
Note that there are other notions of dependency graphs used 
in the literature, most of them,
however, are given as subgraphs of $\depG(\Pi)$.
As we will see later, our definition of dependency graphs 
allows us to draw immediate conclusions for such other notions.


The \emph{incidence graph} of $\Pi$, denoted by $\incG(\Pi)$, is the
undirected graph with vertices $\A(\Pi) \cup \R(\Pi)$ that
contains an edge between a \emph{rule vertex} $r \in \R(\Pi)$ and a
\emph{atom vertex} $a \in \A(\Pi)$ whenever $a \in h(r) \cup p(r)\cup n(r)$. 
The \emph{signed incidence graph} of $\Pi$, denoted by $\sincG(\Pi)$,
is the graph $\incG(\Pi)$, where addionally every edge of
$\incG(\Pi)$ between an atom $a$ and a rule $r$ is annotated with a
label from $\{h,p,n\}$ depending on whether $a$ occurs in $h(r)$,
$p(r)$, or $n(r)$.



\section{Directed Width Measures}\label{sec:dwm}


\begin{figure}[t]
  \centering
    \begin{tikzpicture}[node distance=\sv{3mm}\lv{6mm}, >=latex]
      \tikzstyle{every edge}=[draw,line width=1pt]

      \tikzstyle{mn}=[draw,ellipse,inner sep=0.5mm, line width=1pt]

      \draw

      node[mn] (ucr) {undirected cycle-rank}
      node[below=\sv{1mm}\lv{2mm} of ucr] (und) {}
      node[mn, left=\sv{1cm}\lv{2cm} of und] (td) {treedepth}
      node[mn, below=of td] (upw) {pathwidth}
      node[mn, below=of upw] (utw) {treewidth}
      
      node[mn, right=\sv{1.0cm}\lv{2cm} of und] (cr) {cycle-rank}
      node[mn, below=of cr] (pw) {directed pathwidth}
      node[below=\sv{1mm}\lv{2mm} of pw] (dkd) {}
      node[mn, left=\sv{1cm}\lv{2cm} of dkd] (daw) {DAG-width}
      node[mn, right=\sv{1cm}\lv{2cm} of dkd] (kw) {Kelly-width}
      node[mn, below=\sv{1mm}\lv{2mm} of dkd] (dtw) {directed treewidth}
      node[mn, below=of dtw] (dw) {D-width}

      (ucr) edge[->] (cr)
      (ucr) edge[->] (td)
      (td) edge[->] (upw)
      (upw) edge[->] (utw)
      (cr) edge[->] (pw)
      (pw) edge[->] (daw)
      (pw) edge[->] (kw)
      (daw) edge[->] (dtw)
      (kw) edge[->] (dtw)
      (dtw) edge[->] (dw)
      ;
    \end{tikzpicture}
\sv{\vspace{-8mm}}
  \caption{Propagation of hardness results for the considered width measures. An arc $(A,B)$ indicates that any hardness result parameterized by measure $A$ implies a corresponding hardness result parameterized by $B$.}
  \label{fig:prophard}
\end{figure}
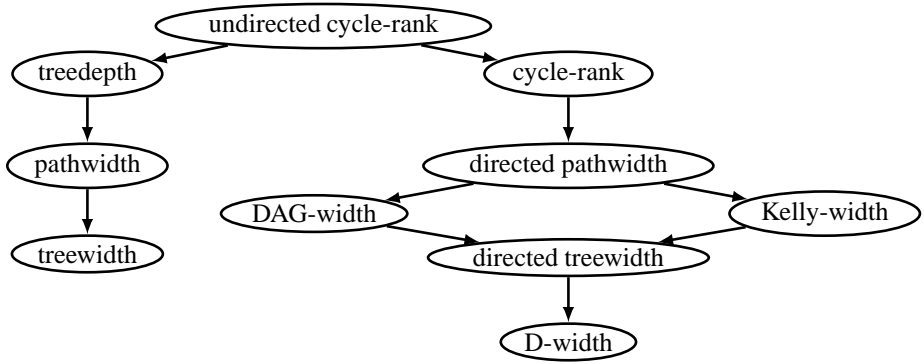

Since many 
representations of ASP programs are
in terms of directed graphs, it is natural to consider parameters for
ASP that are tailor-made for directed graphs. Over the past two
decades various width measures for directed graphs have been
introduced, which are better suited for directed graphs than treewidth, on which they are based.
The most prominent of those are directed
treewidth~\cite{JohnsonRobertsonSeymour01}, directed
pathwidth~\cite{Barat04},
DAG-width~\cite{BerwangerDawarHunterKreutzerObdrzalek12}, 
Kelly-width~\cite{HunterKreutzer08}, and D-width~\cite{Safari05} 
(see also~\cite{KreutzerOrdyniak14}).
Since these width measures are usually smaller on
directed graphs than treewidth, it is worth considering them for
problems that have already been shown to be fixed-parameter tractable
parameterized by treewidth. In particular, all of these measures are zero on directed
acyclic graphs (DAGs), but the treewidth of DAGs can be arbitrary high.
\lv{Moreover, since these measures are based on some notion of
``closeness'' to acyclicity and the complexity of ASP is closely
related to the ``cyclicity'' of the logical rules, one would consider such measures as
promising for obtaining efficient algorithms for ASP.}

In this section, we give results for directed width measures
when 
applied to dependency graphs as defined in Section~\ref{sec:prel}.
To state our results in the most general manner, we will employ the
parameter cycle-rank~\cite{Cohen68}.
Since 
the
cycle-rank is always greater or equal to any of the above mentioned
directed width measures \cite{Gruber12,KaiserKreutzerRabinovichSiebertz14}, 
any (parameterized) hardness result
obtained for cycle-rank carries over to the aforementioned width
measures for directed graphs.



\begin{definition}\label{def:cyclerank}
  Let $D=(V,A)$ be a directed graph. The \emph{cycle-rank} of $D$, denoted by
  $\cyclerank(D)$, is inductively defined as follows: if $D$ is
  acyclic, then $\cyclerank(D)=0$. Moreover, if $D$ is strongly
  connected, then $\cyclerank(D)=1+\min_{v \in V}\cyclerank(D\setminus
  \{v\})$. Otherwise the cycle-rank of $D$ is the maximum cycle-rank of
  any strongly connected component of $D$.
\end{definition}
We will also consider a natural ``undirected version'' of the
cycle-rank for directed graphs, i.e.,
we define the \emph{undirected cycle-rank} of a directed graph $D$, denoted by $\uncyclerank(D)$, to be
the cycle-rank of $\undG(D)$. It is also well known (see,
e.g., \cite{GiannopoulouHunterThilikos12}) that the cycle-rank of
$\undG(D)$ is equal to the treedepth of $\biG(D)$, i.e., the
underlying undirected graph of $D$, and that the treedepth is always an upper bound for the
pathwidth and the treewidth of an undirected
graph~\cite{BodlaenderGilbertHafsteinssonKloks95}. 
Putting these
facts together implies that any hardness result obtained for 
the undirected cycle-rank implies hardness for pathwidth,
treewidth, treedepth as well as the aforementioned directed width measures. See also Figure~\ref{fig:prophard} for an
illustration how hardness results for the considered width measures propagate.

Finally, we would like to remark that both the cycle-rank and the
undirected cycle-rank are easily seen to be closed under taking subgraphs, i.e., the
(undirected) cycle-rank of a graph is always larger or equal to the
(undirected) cycle-rank of every subgraph of the graph.

\subsection*{Hardness Results}\label{sec:hardness}

We show  that
ASP consistency remains as hard as in the general setting even for 
instances that have a dependency graph of constant width in terms of any of the directed width
measures introduced.

For our hardness results, we employ the reduction
given 
in~\cite{EiterGottlob95b} 
showing that ASP consistency is
$\SigmaP{2}$\hy hard in 
general.
The reduction is given from the validity problem for quantified
Boolean formulas (QBF) of the form:
  $\Phi := \exists x_1 \cdots \exists x_n \forall y_1 \cdots \forall y_m \bigvee_{j=1}^rD_j$
where each $D_j$ is a conjunction of at most three literals over the
variables $x_1,\dotsc,x_n$ and $y_1,\dotsc,y_m$. We will denote the
set of all QBF formulas of the above form in the following by $\QBFF$.

Given $\Phi\in\QBFF$, a program $\Pi(\Phi)$ is constructed
as follows. The atoms of $\Pi(\Phi)$ are $x_1,v_1,\dotsc,x_n,v_n$,
$y_1,z_1,\dotsc,y_m,z_m$, and $w$ and $\Pi(\Phi)$ contains the
following rules:
\begin{itemize}
\item for every $i$ with $1\leq i \leq n$, the rule $x_i\vee v_i \leftarrow$,
\item for every $i$ with $1\leq i \leq m$, the rules $y_i\vee z_i
  \leftarrow$, $y_i \leftarrow w$, $z_i \leftarrow w$, and $w
  \leftarrow y_i,z_i$,
\item for every $j$ with $1\leq j \leq r$, the rule 
  $w \leftarrow \sigma(L_{j,1}),\sigma(L_{j,2}),\sigma(L_{j,3})$,
  where $L_{j,l}$ (for $l \in \{1,2,3\}$) is the $l$-th literal that
  occurs in $D_j$ (if $\lvert D_j \rvert < 3$, the respective parts are omitted) and the function $\sigma$ is defined by setting
$\sigma(L)$
to $v_i$ if $L = \neg x_i$,
to $z_i$ if $L = \neg y_i$, and
to $L$ otherwise.
\item the rule $\leftarrow \neg w$ (i.e., with an empty disjunction in the head).
\end{itemize}
It has been shown~\cite[Theorem 38]{EiterGottlob95b} that  a $\QBFF$
formula $\Phi$ is valid 
iff
$\Pi(\Phi)$ has an answer set. 
As checking validity of $\QBFF$ formulas 
is $\SigmaP{2}$\hy complete~\cite{StockmeyerMeyer73}, this reduction shows that
ASP is $\SigmaP{2}$\hy hard.

\begin{figure}[t]
  \centering
  \begin{tikzpicture}[node distance=5mm and 1cm, >=latex]
    \tikzstyle{every edge}=[draw,line width=1pt,latex'-latex']

    \tikzstyle{mn}=[draw,circle,inner sep=1.5mm, line width=1pt]

    \draw
    node[mn, label=above:$x_1$] (x1) {}
    node[mn, right=of x1, label=above:$v_1$] (v1) {}

    node[mn, below=of x1, label=above:$x_2$] (x2) {}
    node[mn, right=of x2, label=above:$v_2$] (v2) {}

    (x1) edge (v1)
    (x2) edge (v2)

    node[mn, node distance=3cm, right=of v1, label=above:$y_1$] (y1) {}
    node[mn, right=of y1, label=above:$z_1$] (z1) {}

    node[mn, below=of y1, label=above:$y_2$] (y2) {}
    node[mn, right=of y2, label=above:$z_2$] (z2) {}

    (y1) edge (z1)
    (y2) edge (z2)

    node[below=of v2] (wd) {}

    node[mn, node distance=1.5cm, right=of wd, label=below:$w$] (w) {}

    (y1) edge[] (w)
    (y2) edge[] (w)
    (z1) edge[out=-45, in=-45] (w)
    (z2) edge[] (w)

    (w) edge[out=225, in=225] (x1)
    (w) edge[] (v2)
    ;
  \end{tikzpicture}
\sv{\vspace{-1.5cm}}
  \caption{The symmetric closure of the dependency graph of the
    program $\Pi(\Phi)$ for the formula 
    $\Phi:=\exists x_1 \exists x_2 \forall y_1 \forall y_2 (x_1\wedge
    \neg y_2) \vee (\neg x_2 \wedge y_2)$. Here $\Pi(\Phi)$ contains the rules $x_i \vee v_i
    \leftarrow$, $y_i \vee z_i \leftarrow$, $y_i\leftarrow w$, $z_i
    \leftarrow w$, for every $i \in \{1,2\}$ and the rules $w \leftarrow x_1, z_2$, $w
    \leftarrow v_2,y_2$.}
  \label{fig:depGQBF}
\end{figure}
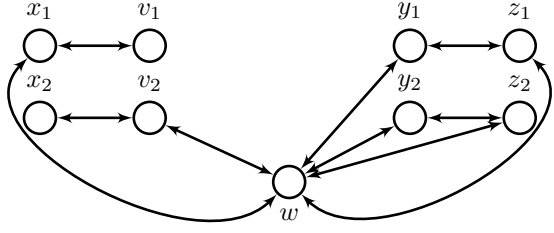

\begin{lemma}
  Let $\Phi$ be a $\QBFF$, then
  $\uncyclerank(\depG(\Pi(\Phi))) \leq 2$.
\end{lemma}
\begin{proof}
  Figure~\ref{fig:depGQBF} illustrates the symmetric closure of 
  $\depG(\Pi(\Phi))$
  for a simple $\QBFF$ formula $\Phi$. As this example
  illustrates, the only arcs in $\undG(\depG(\Pi(\Phi)))$ not incident to $w$ are the
  arcs incident to $x_i$ and $v_i$ and the arcs incident to $y_j$ and $z_j$, for
  $1 \leq i \leq n$ and $1\leq j \leq m$. Hence, after
  removing $w$ from $\undG(\depG(\Pi(\Phi)))$, every strongly connected
  component of the remaining graph
  contains at most two vertices and each of those has hence cycle-rank
  at most one. It follows that the cycle-rank of
  $\undG(\depG(\Pi(\Phi)))$ and hence the undirected cycle-rank of
  $\depG(\Pi(\Phi))$ is at
  most two.
\end{proof}
Together with our considerations from 
above, we obtain:
\begin{THE}
  ASP consistency is $\SigmaP{2}$\hy complete even for instances whose dependency
  graph has width at most two for any of the following width measures:
  undirected cycle-rank, pathwidth, treewidth, treedepth, cycle-rank, directed treewidth, directed
  pathwidth, DAG-width, Kelly-width, and D-width.
\end{THE}
Observe that because the undirected cycle-rank is closed under taking
subgraphs and we chose the ``richest'' variant of the dependency graph, the above result carries over
to the other notions of dependency graphs of ASP programs considered in the literature.

The above result draws a very negative picture of the complexity of
ASP w.r.t. restrictions on the dependency graph. 
In particular, not
even structural restrictions of the dependency graph by 
the usually very successful parameter treewidth can be employed for
ASP. This is in contrast to our second graphical representation of ASP,
the incidence graph, for which it is known that ASP is fixed-parameter
tractable parameterized by the treewidth~\cite{JaklPichlerWoltran09}.
It is hence natural to ask whether the
same still holds under restrictions provided by one of the directed
width measures 
under consideration.
We
first need to discuss how to obtain a directed version of the usually
undirected incidence graph. For this, observe that the
incidence graph, unlike the signed incidence graph, 
provides merely an incomplete model
of the underlying ASP instance. Namely, it misses the information
about \emph{how} atoms occur in rules, i.e., whether they occur in the head, in the
positive body, or in the negative body of a rule.
A directed version of the incidence graph should therefore use the additional
expressiveness provided by the direction of the arcs to incorporate
the
information given by the labels of the signed
incidence graph.
For instance, a natural directed version of the incidence graph
could orient the edges 
depending on
whether an atom occurs in the head or in the body of a rule.
Clearly, there are many 
ways to orient the
edges and it is not a priori clear which of those orientations leads to a directed
version of the incidence graph that is best suited for an application of the
directed width measures. Every 
orientation should, however, be
consistent with the labels of the signed incidence graph, i.e.,
whenever two atoms are connected to a rule via edges having the same
label, 
their arcs should be oriented
in the same way. We call such an orientation of the incidence graph a
\emph{homogeneous} orientation.


\begin{lemma}
  Let $\Phi$ be a $\QBFF$, then
  the cycle-rank of any homogeneous orientation of the incidence graph
  of $\Pi(\Phi)$ is at most one.
\end{lemma}
\begin{proof}
  Let $D$ be a homogeneous orientation of
  $\incG(\Pi(\Phi))$ and let $G = \sincG(\Pi(\Phi)$.
  First observe that in $G \setminus \{w\}$ every rule
  vertex is either only incident to edges with label $h$ or
  to edges of label $p$. Hence, as $D$ is a homogeneous
  orientation, we obtain that every rule vertex of $D \setminus \{w\}$
  is either a source vertex (i.e., having only outgoing arcs) or a
  sink vertex (i.e., having only incoming arcs). So
  $D\setminus \{w\}$ cannot contain a cycle through a rule
  vertex. However, since there are no arcs between atom vertices in
  $D$, we obtain that $D\setminus \{w\}$ is acyclic, which shows
  that the cycle-rank of $D$ is at most one. 
\end{proof}

We can thus state the following result:

\begin{THE}
  ASP consistency is $\SigmaP{2}$\hy complete even for instances whose directed incidence
  graph has width at most one for any of the following width measures:
  cycle-rank, directed treewidth, directed
  pathwidth, DAG-width, Kelly-width, and D-width.
\end{THE}

\section{Clique-Width}

The results in~\cite{GottlobPichlerWei10} imply that bounding the clique-width of the signed incidence graph of a program leads to tractability.

\begin{proposition}
For a program $\Pi$ such that the clique-width of its signed incidence graph is bounded by a constant, we can decide in linear time whether $\Pi$ has an answer set.
\end{proposition}

This result has been established via a formulation of ASP consistency as an MSO$_1$ formula.
Formulating a problem in this logic automatically gives us an FPT algorithm.
However, such algorithms are primarily of theoretical interest due to huge constant factors, and for actually solving problems, it is preferable to explicitly design dynamic programming algorithms~\cite{Cygan15}.

Since our main tractability result considers the clique-width of an
edge-labeled graph, i.e., the signed incidence graph, we will
introduce clique-width for edge-labeled graphs. This definition also
applies to graphs without edge-labels by considering all edges
to be labeled with the same label.
A \emph{$k$\hy graph}, for $k > 0$, is a graph whose
vertices are labeled by integers from
$\{1,\dots,k\}=:[k]$. Additionally, we also allow for the edges of a
$k$-graph to be labeled by some arbitrary but finite set of labels
(in our case the labels will correspond to the signs of the signed incidence graph).
The labeling of the vertices of a graph $G=(V,E)$ is formally denoted
by a function $\L: V \rightarrow [k]$.
We consider an arbitrary graph as a $k$\hy graph with all vertices labeled by~$1$.
We call the $k$\hy graph consisting of exactly one vertex~$v$ (say,
labeled by $i\in [k]$) an \emph{initial $k$\hy graph} and
denote it by~$i(v)$.

Graphs can be constructed from initial $k$\hy graphs by means of
repeated application of the following three operations.
\begin{itemize}
\item \emph{Disjoint union} (denoted by $\oplus$);
\item \emph{Relabeling}: changing all labels $i$ to $j$ (denoted by
  $\rho_{i\rightarrow j}$);
\item \emph{Edge insertion}: connecting all vertices labeled by $i$
  with all vertices labeled by $j$ via an edge with label $\ell$
  (denoted by
  $\edge{i}{j}{\ell}$);
	$i\neq j$; already existing edges are not doubled.
\end{itemize}
A construction of a $k$\hy graph $G$ using the above operations can be represented
by an algebraic term composed of $i(v)$, $\oplus$, $\rho_{i \rightarrow j}$, and
$\edge{i}{j}{\ell}$, ($i,j\in [k]$, and $v$ a vertex
).  Such a term is then called
a \emph{cwd\hy expression defining} $G$.
For any cwd-expression $\sigma$, we use $\L_\sigma: V \to [k]$ to denote the labeling of the graph defined by $\sigma$.
A \emph{$k$\hy expression} is a
$\cwd$\hy expression in which at most $k$ different labels occur.  
The set of all $k$\hy expressions is denoted by $\CWk$.

As an example consider the complete bipartite graph $K_{n,n}$ with
bipartition $A=\{a_1,\dotsc, a_n\}$ and $B=\{b_1,\dotsc, b_n\}$ and
assume that all edges of $K_{n,n}$ are labeled with the label $\ell$. 
A cwd\hy
expression of $K_{n,n}$ using at most two labels is given by the
following steps: (1) introduce all vertices in $A$ using label
$1$, (2) introduce all vertices in $B$ using label $2$, (3) 
take the disjoint union of all these vertices, and (4) add all edges 
between vertices with label $1$ and vertices with label $2$,
i.e., such a cwd\hy expression is given by $\edge{1}{2}{\ell}(1(a_1) \oplus
\dotsb \oplus 1(a_n) \oplus 2(b_1) \oplus \dotsb \oplus 2(b_n))$.
As a second example consider the complete graph $K_n$ on $n$
vertices, where all edges are labeled with label $\ell$. 
A cwd\hy expression for $K_n$ using at most two labels can
be obtained by the following iterative process: Given a cwd\hy
expression $\sigma_{n-1}$ for $K_{n-1}$, where every vertex is labeled with label
$1$, one takes the disjoint
union of $\sigma_{n-1}$ and $2(v)$ (where $v$ is the vertex only contained
in $K_n$ but not in $K_{n-1}$), adds all edges between vertices with
label $1$ and vertices with label $2$, and then relabels label $2$ to
label $1$. Formally, the cwd\hy expression $\sigma_n$ for $K_n$ is given by $(\rho_{2\rightarrow 1}(\edge{1}{2}{\ell}(\sigma_{n-1} \oplus
2(v_2)))$. 

\begin{definition}
The \emph{clique-width} of a graph $G$, $\cwd(G)$, is the smallest integer $k$ such that $G$ can be defined by a $k$-expression.
\end{definition}

Our discussion above thus witnesses that
complete (bipartite) graphs have clique-width $2$. 
Furthermore, 
co-graphs also have 
clique-width $2$ (co-graphs are exactly given by the graphs which
are $P_4$-free, 
i.e., whenever there is a path $(a,b,c,d)$ in the graph
then $\{a,c\}$, $\{a,d\}$ or $\{b,d\}$ is also an edge of the graph)
and
trees have clique-width $3$.

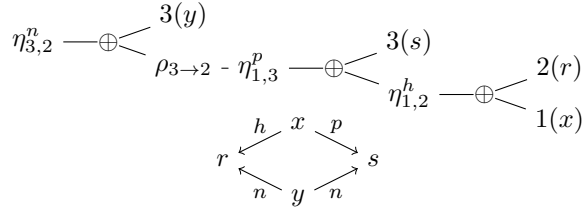
\begin{figure}[t]
\centering
\begin{tikzpicture}[grow=right, level distance=10mm, sibling distance=7mm]
\node {$\edge{3}{2}{n}$}
	child {node [inner sep=0mm] {$\oplus$}
		child {node {$\rho_{3 \rightarrow 2}$}
			child {node {$\edge{1}{3}{p}$}
				child {node [inner sep=0mm] {$\oplus$}
					child {node {$\edge{1}{2}{h}$}
						child {node [inner sep=0mm] {$\oplus$}
							child {node {$1(x)$}}
							child {node {$2(r)$}}
						}
					}
					child {node {$3(s)$}}
				}
			}
		}
		child {node {$3(y)$}}
	};
\end{tikzpicture}

\begin{tikzpicture}[->, scale=1]
\path[use as bounding box] (0,-0.5) rectangle (2,0.2);
\node (r) at (0,0) {$r$};
\node (x) at (1,0.5) {$x$};
\node (y) at (1,-0.5) {$y$};
\node (s) at (2,0) {$s$};
\draw (x) -- (r) node [midway, above, font=\footnotesize] {$h$};
\draw (x) -- (s) node [midway, above, font=\footnotesize] {$p$};
\draw (y) -- (r) node [midway, below, font=\footnotesize] {$n$};
\draw (y) -- (s) node [midway, below, font=\footnotesize] {$n$};
\end{tikzpicture}
\caption{A parse tree (top) of a $3$-expression for $\sincG(\Pi)$ (bottom), where $\Pi$ is the program
consisting of the rules $x \leftarrow \neg y$ and $\leftarrow x, \neg y$}
\label{fig:k-expression}
\end{figure}

We have  already introduced the notion of incidence graphs (resp.\ signed incidence graphs) 
of a program in Section~\ref{sec:prel}. We thus can use cwd-expressions to represent programs.

\begin{example}
\label{ex:k-expression}
Let $\Pi$ be the program with $\A(\Pi) = \{x,y\}$ and $\R(\Pi) = \{r,s\}$,
where $r$ is the rule $x \leftarrow \neg y$ and
$s$ is the rule $\leftarrow x, \neg y$.
Its signed incidence graph $\sincG(\Pi)$ can be constructed by the $3$-expression
$
\edge{3}{2}{n}\Big(\rho_{3 \rightarrow 2}\big(\edge{1}{3}{p}(\edge{1}{2}{h}(1(x) \oplus 2(r)) \oplus 3(s))\big) \oplus 3(y)\Big)$, as depicted in Figure~\ref{fig:k-expression}.
\end{example}

Since every $k$-expression of the signed incidence graph can be
transformed into a $k$-expression of the unsigned incidence graph (by
replacing all operations of the form $\edge{i}{j}{\ell}$ with
$\edge{i}{j}{\alpha}$, where $\alpha$ is new label), it holds that
$\cwd(\incG(\Pi)) \leq \cwd(\sincG(\Pi))$. 
\begin{proposition}\label{pro:unsigned-signed-cw}
  Let $\Pi$ be a program. It holds that $\cwd(\incG(\Pi)) \leq
  \cwd(\sincG(\Pi))$,
  and
  there is a class $\CCC$ of programs 
  such that, for each $\Pi \in \CCC$, $\cwd(\incG(\Pi))=2$ but
  $\cwd(\sincG(\Pi))$ is unbounded.
\end{proposition}
For showing the second statement of the above proposition, consider a
program $\Pi_n$ that has $n^2$ atoms and $n^2$ rules (for some $n
\in \N$),
such that every atom occurs in every rule of $\Pi_n$. Because
the incidence graph is a complete bipartite graph it has
clique-width two and moreover it contains a grid $G$ of size $n\times
n$ as a subgraph. Assume that $\Pi_n$ is defined in such a way that
 an atom $a$ occurring in a rule
$r$ is in the head of $r$ if the edge between $a$ and $r$ occurs in
the grid $G$ and otherwise $a$ is in the (positive) body of $r$.
Then, the clique-width of 
 $\sincG(\Pi_n)$
is at least the clique-width of the 
$n\times n$ grid $G$, which grows with
$n$~\cite{KaminskiLozinMilanic09}. Hence, the class $\CCC$ containing
$\Pi_n$ for every $n \in \N$ shows the second statement of the above proposition.


\subsection{Algorithms}\label{section:algorithms}

In this section, we provide our dynamic programming 
algorithms for deciding existence of an answer set.
We start with 
the classical semantics for programs,
where it is sufficient to just slightly
adapt (a simplified version of) the algorithm for SAT by~\cite{FischerMakowskyRavve06}.
For answer-set semantics, we then 
extend this algorithm in order to deal with the intrinsic higher
complexity of this semantics.

Both algorithms follow the same basic 
principles
by making use of a $k$\hy expression $\sigma$ 
defining a program $\Pi$ 
via its signed incidence graph
in the following way:
We assign certain objects 
to each subexpression of $\sigma$ and
manipulate these objects in a 
bottom-up traversal of the parse tree of the $k$-expression such that the objects in the
root of the parse tree then provide the necessary information to decide
the problem under consideration. The size of these objects
is bounded in terms of $k$ (and independent of the size of $\Pi$) 
and the number of such objects required is linear
in the size of $\Pi$.
Most importantly, we will show that 
these objects can also be efficiently computed 
for bounded $k$. 
Thus, we will obtain the desired 
linear running time.


\subsubsection{Classical Semantics}

\begin{definition}\label{def:kq}
A tuple $Q=(T,F,U)$ with $T,F,U\subseteq [k]$ is called a 
\emph{$k$-triple}, and we refer to its parts using 
$Q_T=T$, $Q_F=F$, and $Q_U=U$. 
The set of all $k$-triples is given by $\mathcal{Q}_k$. 
\end{definition}

The intuition of a triple $(T,F,U)$ is to characterize a set of interpretations $I$ in the following way:
\begin{itemize}
\item For each $i\in T$, at least one atom with label $i$ is true in $I$;
\item for each $i\in F$, at least one atom with label $i$ is false in $I$;
\item for each $i\in U$, there is at least one rule with label $i$ that is ``not satisfied yet''.
\end{itemize}

Formally, the ``semantics'' of a $k$-triple $Q$ with respect to a given program $\Pi$ is given as follows.

\begin{definition}
\label{def:pi-interpretation}
Let $Q\in \mathcal{Q}_k$ and $\Pi$ be a program whose signed incidence graph $(V,E)$ is labeled by $\L: V \rightarrow[k]$.
A $\Pi$-\emph{interpretation} of $Q$
is a set $I \subseteq \A(\Pi)$ that satisfies 
\begin{align*}
Q_T &= \{\L(a) \mid a \in I\},\\
Q_F &= \{\L(a) \mid a \in \A(\Pi) \setminus I\},\mbox{\ and}\\
Q_U &= \{\L(r) \mid r \in \R(\Pi),\; I \notin \mods(r)\}.
\end{align*}
\end{definition}

%

\begin{example}
Consider again program $\Pi$ from Example~\ref{ex:k-expression} and the $3$-expression $\sigma$ from Figure~\ref{fig:k-expression}.
Let $Q$ be the $3$-triple 
$(\{1\}, \{3\}, \{2\})$.
Observe that $\{x\}$ is a $\Pi$-interpretation of $Q$:
It sets $x$ to true and $y$ to false, and $\L_\sigma(x) \in Q_T$ and $\L_\sigma(y) \in Q_F$ hold as required;
the rule $s$ is not satisfied by $\{x\}$, and indeed $\L_\sigma(s) \in Q_U$.
We can easily verify that no other subset of $\A(\Pi)$ is a $\Pi$-interpretation of $Q$:
Each $\Pi$-interpretation of $Q$ must set $x$ to true and $y$ to false, as these are the only atoms labeled with $1$ and $3$, respectively.
\end{example}

%

We use the following notation for $k$-triples $Q$, $Q'$, and
set
$S\subseteq [k]$.
\begin{itemize}
\item $Q\oplus Q' = ( 
	Q_T \cup Q'_T, 
	Q_F \cup Q'_F, 
	Q_U \cup Q'_U )$
\item $Q^{i\rightarrow j} = (
	Q_T^{i\rightarrow j},
	Q_F^{i\rightarrow j},
	Q_U^{i\rightarrow j})$
	where 
	for $S\subseteq [k]$, 
	$$
	S^{i\rightarrow j} = S \setminus \{i\} \cup \{j\}\mbox{\ if $i\in S$ and\ } 
	S^{i\rightarrow j} = S \mbox{\ otherwise.}
	$$	
\item 
$Q^{S,i,j}=
(Q_T,Q_F,Q_U\setminus\{j\})$ if $i\in S$;
$Q^{S,i,j}=Q$ otherwise.
\end{itemize}

Using these abbreviations, we define our dynamic programming algorithm:
We assign to each subexpression $\sigma$ of a given $k$-expression a set of
triples by recursively defining a function $f$, which associates to $\sigma$ a
set of $k$-triples as follows.

\begin{definition}\label{def:f}
The function 
$f: \CWk \rightarrow  2^{\mathcal{Q}_k}$
is recursively defined along the structure of $k$\hy expressions as follows.
\begin{itemize}
 \item $f(i(v))=
	\begin{cases}
		\big\{\big(\{i\},\emptyset,\emptyset\big),\; \big(\emptyset,\{i\},\emptyset\big)\big\} &\text{if $v$ is an atom node} \\
		\big\{\big(\emptyset,\emptyset,\{i\}\big)\big\} &\text{if $v$ is a rule node}
	\end{cases}
	$
 \item $f(\sigma_1 \union \sigma_2)=\{
		Q\oplus Q' 
		 \mid 
			 Q\in f(\sigma_1), Q'\in f(\sigma_2)\}$ 
 \item $f(\relabel{i}{j}(\sigma))=\{  
		Q^{i\rightarrow j}   \mid Q\in  f(\sigma) \}$ 
 \item 
$f(\edge{i}{j}{h}(\sigma))=
f(\edge{i}{j}{n}(\sigma))= 
\{ Q^{Q_T,i,j} \mid Q \in f(\sigma)\}$
\item
$f(\edge{i}{j}{p}(\sigma))=
\{ Q^{Q_F,i,j} \mid Q \in f(\sigma)\}$
\end{itemize}
\end{definition}

\begin{example}
Consider again program $\Pi$ from Example~\ref{ex:k-expression} and the $3$-expression depicted in Figure~\ref{fig:k-expression}.
To break down the structure of $\sigma$,
let $\sigma_1, \ldots, \sigma_6$ be subexpressions of $\sigma$ such that
$\sigma = \edge{3}{2}{n}(\sigma_1)$,
$\sigma_1 = \sigma_2 \oplus 3(y)$,
$\sigma_2 = \rho_{3 \rightarrow 2}(\sigma_3)$,
$\sigma_3 = \edge{1}{3}{p}(\sigma_4)$,
$\sigma_4 = \sigma_5 \oplus 3(s)$,
$\sigma_5 = \edge{1}{2}{h}(\sigma_6)$ and
$\sigma_6 = 1(x) \oplus 2(r)$.
We get
$f(1(x)) = \big\{(\{1\}, \emptyset, \emptyset), (\emptyset, \{1\}, \emptyset)\big\}$ and
$f(2(r)) = \big\{(\emptyset, \emptyset, \{2\})\big\}$.
These sets are then combined to
$f(\sigma_6) = \big\{(\{1\}, \emptyset, \{2\}), (\emptyset, \{1\}, \{2\})\big\}$.
The program defined by $\sigma_6$ consists of atom $x$ and rule $r$, but $x$ does not occur in $r$ yet.
Accordingly, the $k$-triple $(\{1\}, \emptyset, \{2\})$ models the situation where $x$ is set to true, which does not satisfy $r$ (since the head and body of $r$ are still empty), hence the label of $r$ is in the last component;
the $3$-triple $(\emptyset, \{1\}, \{2\})$ represents $x$ being set to false, which does not satisfy $r$ either.
Next, $\sigma_5$ causes all atoms with label $1$ (i.e., just $x$) to be inserted into the head of all rules with label $2$ (i.e., just $r$), and we get
$f(\sigma_5) = \big\{(\{1\}, \emptyset, \emptyset), (\emptyset, \{1\}, \{2\})\big\}$.
We obtain the first element $(\{1\}, \emptyset, \emptyset) = Q^{Q_T,1,2}$ from $Q = (\{1\}, \emptyset, \{2\})$ by
removing the label $2$ from $Q_U$ because $1 \in Q_T$.
The idea is that the heads of all rules labeled with $2$ now contain all atoms labeled with $1$, so these rules become satisfied by every interpretation that sets some atom labeled with $1$ to true.
Next, $\sigma_4$ adds the rule $s$ with label $3$ and we get
$f(\sigma_4) = \big\{(\{1\}, \emptyset, \{3\}), (\emptyset, \{1\}, \{2,3\})\big\}$.
The edge added by $\sigma_3$ adds all atoms with label $1$ (i.e., just $x$) into the positive body of all rules with label $3$ (i.e., just $s$),
which results in
$f(\sigma_3) = \big\{(\{1\}, \emptyset, \{3\}), (\emptyset, \{1\}, \{2\})\big\}$.
Observe that the last component of the second element no longer contains $3$,
i.e., setting $x$ to false makes $s$ true.
Now the label $3$ is renamed to $2$, and we get
$f(\sigma_2) = \big\{(\{1\}, \emptyset, \{2\}), (\emptyset, \{1\}, \{2\})\big\}$.
Note that now $r$ and $s$ are no longer distinguishable since they now share the same label.
Hence all operations that add edges to $r$ will also add edges to $s$ and vice versa.
In $\sigma_1$, atom $y$ is added with label $3$ and we get four $3$-triples in $f(\sigma_1)$:
From $(\{1\}, \emptyset, \{2\})$ in $f(\sigma_2)$ we obtain $(\{1,3\}, \emptyset, \{2\})$ and $(\{1\}, \{3\}, \{2\})$,
and from $(\emptyset, \{1\}, \{2\})$ in $f(\sigma_2)$ we get $(\{3\}, \{1\}, \{2\})$ and $(\emptyset, \{1,3\}, \{2\})$.
In $\sigma$, we add a negative edge from all atoms labeled with $3$ (i.e., just $y$) to all rules labeled with $2$ (both $r$ and $s$).
From $(\{1,3\}, \emptyset, \{2\})$ in $f(\sigma_1)$ we now get $(\{1,3\}, \emptyset, \emptyset)$,
from $(\{3\}, \{1\}, \{2\})$ we get $(\{3\}, \{1\}, \emptyset)$,
and the $3$-triples $(\{1\}, \{3\}, \{2\})$ and $(\emptyset, \{1,3\}, \{2\})$ from $f(\sigma_1)$ occur unmodified in $f(\sigma)$.
As we will prove shortly, for each $k$-triple $Q$ in $f(\sigma)$, there is a $\Pi$-interpretation of $Q$.
So if there is a $k$-triple $Q$ in $f(\sigma)$ such that $Q_U = \emptyset$, then $\Pi$ has a classical model due to the definition of $Q_U$.
For instance, $(\{1,3\}, \emptyset, \emptyset)$ has a $\Pi$-interpretation $\{x,y\}$, which is obviously a model of $\Pi$.
\end{example}

We now prove correctness of our algorithm:

\begin{lemma}
\label{lem:f-sound-and-complete}
Let $\Pi$ be a program and $\theta$ be a $k$-expression for $\sincG(\Pi)$.
For every set $I \subseteq \A(\Pi)$, there is a $k$-triple $Q \in f(\theta)$ such that
$I$ is a $\Pi$-interpretation of $Q$, and
for every $k$-triple $Q \in f(\theta)$ there is a set $I \subseteq \A(\Pi)$ such that
$I$ is a $\Pi$-interpretation of $Q$.
\end{lemma}

\begin{proof}
We prove the first statement by induction on the structure of a $k$-expression $\theta$ defining $\Pi$.
Let $\sigma$ be a subexpression of $\theta$, let $\Pi_\sigma$ denote the program defined by $\sigma$, and let $I \subseteq \A(\Pi_\sigma)$.

\medskip\noindent
If $\sigma = i(r)$, for $r \in \R(\Pi)$, then
$\A(\Pi_\sigma) = \emptyset$, so $I = \emptyset$.
Moreover, $\R(\Pi_\sigma)$ consists of an unsatisfiable rule (its head and body are empty).
Hence $I$ is a $\Pi_\sigma$-interpretation of 
$(\emptyset, \emptyset, \{i\})$ in $f(\sigma)$.

\medskip\noindent
If $\sigma = i(a)$, for $a \in \A(\Pi)$, then
$\A(\Pi_\sigma) = \{a\}$ and $\R(\Pi_\sigma) = \emptyset$.
If $I = \emptyset$, then $I$ is a $\Pi_\sigma$-interpretation of the $k$-triple $(\emptyset, \{i\}, \emptyset)$ in $f(\sigma)$.
Otherwise $I = \{a\}$ and $I$ is a $\Pi_\sigma$-interpretation of the $k$-triple $(\{i\}, \emptyset, \emptyset)$ in $f(\sigma)$.

\medskip\noindent
If $\sigma = \sigma_1 \oplus \sigma_2$, let
$i \in \{1,2\}$,
$\Pi_i = \Pi_{\sigma_i}$ and
$I_i = I \cap \A(\Pi_i)$.
By definition of $\Pi_i$, it holds that $\A(\Pi) = \A(\Pi_1) \cup \A(\Pi_2)$, $\R(\Pi) = \R(\Pi_1) \cup \R(\Pi_2)$ and $I = I_1 \cup I_2$.
By induction hypothesis, $I_i$ is a $\Pi_i$-interpretation of some $k$-triple $Q_i$ in $f(\sigma_i)$.
By definition of $f$, there is a $k$-triple $Q$ in $f(\sigma)$ with
$Q_T = {Q_1}_T \cup {Q_2}_T$,
$Q_F = {Q_1}_F \cup {Q_2}_F$ and
$Q_U = {Q_1}_U \cup {Q_2}_U$.
This allows us to easily verify that $I$ is a $\Pi_\sigma$-interpretation of $Q$ by checking the conditions in Definition~\ref{def:pi-interpretation}.

\medskip\noindent
If $\sigma = \relabel{i}{j}(\sigma')$, then
$\Pi_\sigma = \Pi_{\sigma'}$
and, by induction hypothesis, $I$ is a $\Pi_{\sigma'}$-interpretation of some $k$-triple $Q$ in $f(\sigma')$.
By definition of $f$, the $k$-triple $Q^{i\rightarrow j}$ in $f(\sigma)$ is the result of replacing $i$ by $j$ in each of $Q_T$, $Q_F$ and $Q_U$.
Hence we can easily verify that $I$ satisfies all conditions for being a $\Pi_\sigma$-interpretation of $Q^{i\rightarrow j}$.

\medskip\noindent
If $\sigma = \edge{i}{j}{\ell}(\sigma')$, for $\ell \in \{h,n\}$, then
$\A(\Pi_\sigma) = \A(\Pi_{\sigma'})$.
Hence, by induction hypothesis, $I$ is a $\Pi_{\sigma'}$-interpretation of some $k$-triple $Q'$ in $f(\sigma')$.
We use $Q$ to denote the $k$-triple $Q'^{Q'_T,i,j}$, which is in $f(\sigma)$.
Since $Q_T = Q'_T$, $Q_F = Q'_F$
and $\L_\sigma = \L_{\sigma'}$,
$I$ satisfies the first two conditions for being a $\Pi_\sigma$-interpretation of $Q$.
It remains to check the third condition.

For every $j' \in [k] \setminus \{j\}$ it holds that
$j' \in Q_U$ if and only if $j' \in Q'_U$.
By induction hypothesis,
the latter is the case if and only if
there is a rule $r' \in \R(\Pi_{\sigma'})$ such that
$\L_{\sigma'}(r') = j'$ and $I \notin \mods(r')$.
This is equivalent to the existence of a rule
$r \in \R(\Pi_\sigma)$ 
such that
$\L_\sigma(r) = j'$,
$h(r) = h(r')$,
$p(r) = p(r')$,
$n(r) = n(r')$ and
$I \notin \mods(r)$,
since  
$\sincG(\Pi_\sigma)$
only differs from 
$\sincG(\Pi_{\sigma'})$
by additional edges that are 
not incident to $r$ due to $j' \neq j$.

It remains to check that $j \in Q_U$ if and only if
there is a rule $r \in \R(\Pi_\sigma)$ such that
$\L_\sigma(r) = j$ and $I \notin \mods(r)$.
First suppose toward a contradiction that
$j \in Q_U$ while
$I$ is a model of every rule $r \in \R(\Pi_\sigma)$ such that $\L_\sigma(r) = j$.
Since $Q_U \subseteq Q'_U$,
also $j \in Q'_U$
and by induction hypothesis
there is a rule $r' \in \R(\Pi_{\sigma'})$ such that $\L_{\sigma'}(r') = j$ and $I$ is not a model of $r'$.
There is a corresponding rule $r \in \R(\Pi_\sigma)$,
for which $\L_\sigma(r) = j$, $h(r') \subseteq h(r)$, $n(r') \subseteq n(r)$ and $p(r') = p(r)$ hold.
Since $I$ is a model of $r$ but not of $r'$,
$I$ contains some atom labeled with $i$ (by both $\L_{\sigma'}$ and $\L_{\sigma}$) because
all atoms in $h(r) \setminus h(r')$ and $n(r) \setminus n(r')$ are labeled with $i$.
By induction hypothesis, this implies $i \in Q'_T$,
which leads to the contradiction $j \notin Q_U$ by construction of $f$.

Finally, suppose toward a contradiction that
$j \notin Q_U$ and there is a rule $r \in \R(\Pi_\sigma)$ such that
$\L_\sigma(r) = j$ and $I \notin \mods(r)$.
The rule $r'$ corresponding to $r$ in $\Pi_{\sigma'}$ with $\L_{\sigma'}(r') = j$ is not satisfied by $I$ either,
since $h(r') \subseteq h(r)$, $n(r') \subseteq n(r)$ and $p(r') = p(r)$.
By induction hypothesis, this entails $j \in Q'_U$.
Due to $j \notin Q_U$, it holds that $i \in Q'_T$, so
there is an $a \in I$ with $\L_{\sigma'}(a) = \L_\sigma(a) = i$.
Due to the new edge from $a$ to $r$, either
$a \in h(r)$ or $a \in n(r)$.
This yields the contradiction that $I$ is a model of $r$.

\medskip\noindent
The case $\sigma = \edge{i}{j}{p}(\sigma')$ is symmetric.

\medskip\noindent
The proof of the second statement is similar.
\end{proof}


We can now state our FPT result for classical models:

\begin{THE}
\label{thm:classical-models-fpt}
Let $k$ be an integer and $\Pi$ be a program.
Given a $k$-expression for the signed incidence graph of $\Pi$, we can decide in linear time whether $\Pi$ has a model.
\end{THE}

\begin{proof}
Let $k$ be a constant, $\Pi$ be a program and $\sigma$ be a $k$-expression of $\sincG(\Pi)$.
We show that there is a model of $\Pi$
if and only if there is a $k$-triple $Q$ in $f(\sigma)$ with $Q_U=\emptyset$:
If $\Pi$ has a model $I$, then $I$ is a $\Pi$-interpretation of a $k$-triple $Q$ in $f(\sigma)$, by Lemma~\ref{lem:f-sound-and-complete},
and $Q_U = \emptyset$ by Definition~\ref{def:pi-interpretation}.
Conversely, if there is a $k$-triple $Q$ in $f(\sigma)$ with $Q_U = \emptyset$, then
there is a $\Pi$-interpretation $I$ of $Q$, by Lemma~\ref{lem:f-sound-and-complete},
and $Q_U = \emptyset$ implies that $I$ is a model of $\Pi$ by Definition~\ref{def:pi-interpretation}.
Finally, it is easy to see that $f(\sigma)$ can be computed in linear time.
\end{proof}

\subsubsection{Answer-Set Semantics}

For full disjunctive ASP we need a more involved data structure.

\begin{definition}\label{def:kqg}
A pair $(Q,\Gamma)$ with 
with $Q\in\mathcal{Q}_k$ and 
$\Gamma\subseteq \mathcal{Q}_k$ 
is called a 
\emph{$k$-pair}. 
The set of all $k$-pairs is given by $\mathcal{P}_k$. 
\end{definition}


Given a $k$-pair $(Q,\Gamma)$, the purpose of $Q$ is, as for classical semantics, to represent $\Pi$-interpretations $I$ (that in the end correspond to models).
Every $k$-triple in $\Gamma$ represents
sets $J$ of atoms such that $J \subset I$.
If, in the end, there is such a set $J$ that still satisfies every rule in the reduct w.r.t.\ $I$,
then we conclude that $I$ is not an answer set.

\begin{definition}
\label{def:pi-i-interpretation}
Let $Q\in \mathcal{Q}_k$,
let $\Pi$ be a program whose signed incidence graph $(V,E)$ is labeled by $\L: V \rightarrow[k]$,
and let $I \subseteq \A(\Pi)$.
A $\Pi^I$-\emph{interpretation} of $Q$
is a set $J \subseteq \A(\Pi)$ such that
\begin{align*}
Q_T &= \{\L(a) \mid a \in J\},\\
Q_F &= \{\L(a) \mid a \in \A(\Pi) \setminus J\}, \mbox{\ and}\\
Q_U &= \{\L(r) \mid r \in \R(\Pi),\; n(r) \cap I = \emptyset,\; J \notin \mods(r^+)\}.
\end{align*}
\end{definition}

We can now define our dynamic programming algorithm for ASP:

\begin{definition}\label{def:g}
The function 
$g: \CWk \rightarrow  2^{\mathcal{P}_k}$
is recursively defined along the structure of $k$\hy expressions as follows.

\begin{itemize}
 \item $g(i(v)) = \Big\{ 
			\big((\{i\},\emptyset,\emptyset), \; \big\{(\emptyset,\{i\},\emptyset)\big\} \big),
			\quad
			\big((\emptyset,\{i\},\emptyset),\;\emptyset \big)
		\Big\}$
\sv{		\\}
if $v$ is at atom node
 \item $g(i(v)) = \Big\{\big((\emptyset,\emptyset,\{i\}),\; \emptyset\big)\Big\}$ if $v$ is a rule node
 \item $g(\sigma_1 \union \sigma_2)=\{ \big(
		Q_1 \oplus Q_2, R_{Q_1, Q_2, \Gamma_1, \Gamma_2} \mid (Q_i,\Gamma_i)\in g(\sigma_i)\big)\}$,
		where
		$R_{Q_1, Q_2, \Gamma_1, \Gamma_2} = \big\{
			S_1 \oplus S_2 \mid S_i \in \Gamma_i
			\big\} \cup \big\{
			Q_1 \oplus S \mid S \in \Gamma_2
			\big\} \cup \big\{
			S \oplus Q_2 \mid S \in \Gamma_1
			\big\}
			\big)$
 \item $g(\relabel{i}{j}(\sigma))=\{ \big(( 
		Q^{i\rightarrow j}, 
		\{ 
			R^{i\rightarrow j} \mid R\in\Gamma\} \big)
 			\mid (Q,\Gamma) \in  g(\sigma) \}$ 
 \item 
$g(\edge{i}{j}{h}(\sigma))=
\{ \big( Q^{Q_T,i,j} , \{ R^{R_T,i,j} 
		\mid R\in \Gamma \}\big) \mid (Q,\Gamma)\in g(\sigma)\}$
 \item 
$g(\edge{i}{j}{p}(\sigma))=
\{ \big( Q^{Q_F,i,j} , \{ R^{R_F,i,j} 
		\mid R\in \Gamma \}\big) \mid (Q,\Gamma)\in g(\sigma)\}$
 \item 
$g(\edge{i}{j}{n}(\sigma))=
\{ \big( Q^{Q_T,i,j} , \{ R^{Q_T,i,j}
		\mid R\in \Gamma \}\big) \mid (Q,\Gamma)\in g(\sigma)\}$
\end{itemize}
\end{definition}

Note the use of $Q_T$ in $R^{Q_T,i,j}$ in the definition of 
$g(\edge{i}{j}{n}(\sigma))$:
Whenever an interpretation $I$ represented by $Q$ sets an atom from the negative body of a rule $r$ to true,
the rule $r$ has no counterpart in the reduct w.r.t.\ $I$,
so, for each subset $J$ of $I$, we remove $r$ from the set of rules whose counterpart in the reduct is not yet satisfied by $J$.



\begin{example}
Let $\Pi$ be the program consisting of a single rule $\leftarrow \neg x$, which we denote by $r$,
and let $\sigma = \edge{1}{2}{n}(1(x) \oplus 2(r))$.
Let $(Q,\Gamma)$ be the $k$-pair in $g(1(x))$ with
$Q = (\{1\}, \emptyset, \emptyset)$ and $\Gamma = \{(\emptyset, \{1\}, \emptyset)\}$.
The $k$-triple $Q$ represents the set of atoms $\{x\}$.
Since this set has the proper subset $\emptyset$, there is a $k$-triple in $\Gamma$ that indeed corresponds to this subset.
Now let $(Q,\Gamma) = ((\emptyset, \{1\}, \emptyset), \, \emptyset)$ be the other $k$-pair in $g(1(x))$.
Here $Q$ represents the empty set of atoms, which has no proper subsets, hence $\Gamma$ is empty.
For the single $k$-pair $((\emptyset, \emptyset, \{2\}), \, \emptyset)$ in $g(2(r))$, the situation is similar.
Next, at $g(1(x) \oplus 2(r))$, we combine every $k$-pair $(Q_1,\Gamma_1)$ from $g(1(x))$ with every $k$-pair $(Q_2,\Gamma_2)$ from $g(2(r))$ to a new $k$-pair.
For instance, consider 
$Q_1 = (\{1\}, \emptyset, \emptyset)$ and $\Gamma_1 = \{S\}$, where $S = (\emptyset, \{1\}, \emptyset)$,
as well as  
$Q_2 = (\emptyset, \emptyset, \{2\})$ and $\Gamma_2 = \emptyset$.
By definition of $g$, we obtain a new $k$-pair $(Q,\Gamma)$, where $Q = Q_1 \oplus Q_2 = (\{1\}, \emptyset, \{2\})$,
and $\Gamma$ contains the single element $Q_2 \oplus S = (\emptyset, \{1\}, \{2\})$.
Recall that the purpose of $Q$ is to represent sets of atoms $I$, and each element of $\Gamma$ shall represent proper subsets of $I$;
in this case, $Q$ represents $\{x\}$, and the element $Q_2 \oplus S$ in $\Gamma$ represents the proper subset $\emptyset$.
Next, at $g(\sigma)$ we introduce a negative edge from $x$ to $r$.
From the $k$-pair $(Q, \{S\})$ in $g(1(x) \oplus 2(r))$, where $Q = (\{1\}, \emptyset, \{2\})$ and $S = (\emptyset, \{1\}, \{2\})$,
we obtain the $k$-pair $(Q', \{S'\})$ in $g(\sigma)$,
where
$Q' = Q^{Q_T,i,j} = (\{1\}, \emptyset, \emptyset)$ (i.e., the label $2$ from $Q_U$ has disappeared)
and $S' = S^{Q_T,i,j} = (\emptyset, \{1\}, \emptyset)$.
Here $2$ has disappeared from $S_U$ because the reduct w.r.t.\ all sets of atoms represented by $Q'$ no longer contains any rule labeled with $2$.
Note that the classical model $\{x\}$ represented by $Q'$ is no answer set even though $Q'_U = \emptyset$.
The reason is that $S'$ witnesses (by $S'_U = \emptyset$) that $\emptyset \in \mods(\Pi^{\{x\}})$.
\end{example}

We now prove correctness of the algorithm from Definition~\ref{def:g}.

\begin{lemma}
\label{lem:g-sound-and-complete}
Let $\Pi$ be a program and $\theta$ be a $k$-expression for $\sincG(\Pi)$.
For every set $I \subseteq \A(\Pi)$ there is a $k$-pair $(Q,\Gamma) \in g(\theta)$ such that
(i) $I$ is a $\Pi$-interpretation of $Q$ and
(ii) for every set $J \subset I$ there is a $k$-triple $R \in \Gamma$ such that
$J$ is a $\Pi^I$-interpretation of $R$.
Moreover,
for every $k$-pair $(Q, \Gamma) \in g(\theta)$ there is a set $I \subseteq \A(\Pi)$ such that
(i') $I$ is a $\Pi$-interpretation of $Q$ and
(ii') for each $k$-triple $R \in \Gamma$, there is a set $J \subset I$ such that $J$ is a $\Pi^I$-interpretation of $R$.
\end{lemma}

\begin{proof}
Observe that for each $(Q,\Gamma)$ in $g(\theta)$ it holds that $Q \in f(\theta)$, and for each $Q$ in $f(\theta)$ there is some $(Q,\Gamma)$ in $g(\theta)$.
Hence we can apply the same arguments as in the proof of Lemma~\ref{lem:f-sound-and-complete} for (i) and (i').
In addition, similar arguments can be used within each of the distinguished cases for (ii) and (ii').
We use induction on the structure of a $k$-expression $\theta$ defining $\Pi$.
Let $\sigma$ be a subexpression of $\theta$, let $\Pi_\sigma$ denote the program defined by $\sigma$, and let $I \subseteq \A(\Pi_\sigma)$.
\sv{%
We only prove some of the cases, which should suffice to extend the ideas from Lemma~\ref{lem:f-sound-and-complete} in a similar way to prove the other cases.%
}

\lv{%
\medskip\noindent
If $\sigma = i(r)$, for $r \in \R(\Pi)$, then
$\A(\Pi_\sigma) = \emptyset$, so $I = \emptyset$.
As in the proof of Lemma~\ref{lem:f-sound-and-complete},
we can show that
$I$ is a $\Pi_\sigma$-interpretation of the $k$-triple $Q = (\emptyset, \emptyset, \{i\})$.
Since $I$ has no proper subsets, condition (ii) holds trivially for the $k$-pair $(Q,\emptyset)$ in $g(\sigma)$.

\medskip\noindent
If $\sigma = i(a)$, for $a \in \A(\Pi)$, then
$\A(\Pi_\sigma) = \{a\}$ and $\R(\Pi_\sigma) = \emptyset$.
If $I = \emptyset$, then $I$ is a $\Pi_\sigma$-interpretation of the $k$-triple $Q = (\emptyset, \{i\}, \emptyset)$
and (ii) holds trivially for the $k$-pair $(Q,\emptyset)$ in $g(\sigma)$.
Otherwise $I = \{a\}$ holds.
Let $(Q,\Gamma)$ be the $k$-pair in $g(\sigma)$ for which $Q = (\{i\}, \emptyset, \emptyset)$ and $\Gamma = \{(\emptyset, \{i\}, \emptyset)\}$ hold.
Clearly $I$ is a $\Pi_\sigma$-interpretation of $Q$.
The only proper subset of $I$ is $\emptyset$, which is a $\Pi_\sigma^I$-interpretation of the only element of $\Gamma$.
}

\medskip\noindent
If $\sigma = \sigma_1 \oplus \sigma_2$, let
$\Pi_i = \Pi_{\sigma_i}$ and
$I_i = I \cap \A(\Pi_i)$,
for any $i \in \{1,2\}$.
By definition of $\Pi_i$, it holds that $\A(\Pi) = \A(\Pi_1) \cup \A(\Pi_2)$, $\R(\Pi) = \R(\Pi_1) \cup \R(\Pi_2)$ and $I = I_1 \cup I_2$.
Let $J \subset I$ and, for $i \in \{1,2\}$, let $J_i = J \cap \A(\Pi_i)$.
Observe that $J_1 \subseteq I_1$ and $J_2 \subseteq I_2$, and at least one inclusion is proper.
We distinguish three cases:
\begin{myitemize}
\item
If $J_1 \subset I_1$ and $J_2 \subset I_2$, then,
by induction hypothesis,
for any $i \in \{1,2\}$
there is a $k$-pair $(Q_i,\Gamma_i)$ in $g(\sigma_i)$ such that
$I_i$ is a $\Pi_i$-interpretation of $Q_i$ and
there is a $k$-triple $R_i \in \Gamma_i$ such that $J_i$ is a $\Pi_i^{I_i}$-interpretation of $R_i$.
By definition of $g$, there is a $k$-pair $(Q,\Gamma)$ in $g(\sigma)$ such that
$Q = Q_1 \oplus Q_2$ and
there is a $k$-triple $R$ in $\Gamma$ such that
$R = R_1 \oplus R_2$.
We can easily check that $I$ is a $\Pi_\sigma$-interpretation of $Q$ and that
$J$ is a $\Pi_\sigma^I$-interpretation of $R$.
\item
If $J_1 \subset I_1$ and $J_2 = I_2$, then,
by induction hypothesis,
there is a $k$-pair $(Q_1,\Gamma_1)$ in $g(\sigma_1)$ such that
$I_1$ is a $\Pi_1$-interpretation of $Q_1$ and
there is a $k$-triple $R_1 \in \Gamma_1$ such that $J_1$ is a $\Pi_1^{I_1}$-interpretation of $R_1$.
Also, there is a $k$-pair $(Q_2,\Gamma_2)$ in $g(\sigma_2)$ such that
$I_2$ is a $\Pi_2$-interpretation of $Q_2$.
By definition of $g$, there is a $k$-pair $(Q,\Gamma)$ in $g(\sigma)$ such that
$Q = Q_1 \oplus Q_2$ and
there is a $k$-triple $R$ in $\Gamma$ such that
$R = R_1 \oplus Q_2$.
We can easily check that $I$ is a $\Pi_\sigma$-interpretation of $Q$ and that
$J$ is a $\Pi_\sigma^I$-interpretation of $R$.

\item
The case $J_1 = I_1$, $J_2 \subset I_2$ is symmetric.
\end{myitemize}

\lv{%
\medskip\noindent
We omit the cases 
$\sigma = \relabel{i}{j}(\sigma')$,
$\sigma = \edge{i}{j}{h}(\sigma')$ and
$\sigma = \edge{i}{j}{p}(\sigma')$, as they are do not offer much additional insight, given the proof of Lemma~\ref{lem:f-sound-and-complete} and the following case.

}

\medskip\noindent
If $\sigma = \edge{i}{j}{n}(\sigma')$, then
there is again a $k$-pair $(Q', \Gamma')$ in $g(\sigma')$ such that
(i) $I$ is a $\Pi_{\sigma'}$-interpretation of $Q'$ and
(ii) for each $J \subset I$ there is a $k$-triple $R$ in $\Gamma'$ such that $J$ is a $\Pi_{\sigma'}^I$-interpretation of $R$.
Let $(Q,\Gamma)$ be the $k$-pair in $g(\sigma)$ for which
$Q = Q'^{Q'_T,i,j}$ and
$\Gamma = \{R^{Q'_T,i,j} \mid R \in \Gamma'\}$ hold.
As before, $I$ is a $\Pi_\sigma$-interpretation of $Q$.
Let $J \subset I$, let $R'$ be the $k$-triple in $\Gamma'$ such that $J$ is a $\Pi_{\sigma'}^I$-interpretation of $R'$,
and let $R = R'^{Q'_T,i,j}$.
As before, $J$ satisfies the first two conditions for being a $\Pi_\sigma^I$-interpretation of $R$.
It remains to check the third condition.

For all labels except $j$, we proceed as before.
We now check that $j \in R_U$ if and only if
there is a rule $r \in \R(\Pi_\sigma)$ such that
$\L_\sigma(r) = j$,
$n(r) \cap  I = \emptyset$
and $J \notin \mods(r^+)$.
First suppose toward a contradiction that
$j \in R_U$ while
$J \in \mods(r^+)$ for each $r \in \R(\Pi_\sigma)$ such that $\L_\sigma(r) = j$ and $n(r) \cap I = \emptyset$.
Since $R_U \subseteq R'_U$,
also $j \in R'_U$
and by induction hypothesis
there is a rule $r' \in \R(\Pi_{\sigma'})$ such that
$\L_{\sigma'}(r') = j$,
$n(r') \cap I = \emptyset$
and $J$ is not a model of $r'^+$.
There is a corresponding rule $r \in \R(\Pi_\sigma)$,
for which $\L_\sigma(r) = j$,
$h(r') = h(r)$, $p(r') = p(r)$ and $n(r') \subseteq n(r)$ hold.
Since $J$ is a model of $r^+$ but not of $r'^+$ (and these rules are identical),
there is an $a \in n(r) \cap I$ with $\L_\sigma(a) = i$.
From $a \in I$ it follows by induction hypothesis that
$i \in Q'_T$, but this leads to the contradiction $j \notin R_U$.

Finally, suppose toward a contradiction that
$j \notin R_U$ and there is a rule $r \in \R(\Pi_\sigma)$ such that
$\L_\sigma(r) = j$,
$n(r) \cap  I = \emptyset$
and $J \notin \mods(r^+)$.
Let $r'$ be the rule corresponding to $r$ in $\Pi_{\sigma'}$.
For this rule it holds that $\L_{\sigma'}(r') = j$ and $n(r') \subseteq n(r)$, so $n(r') \cap I = \emptyset$.
The set $J$ does not satisfy $r'^+$ either,
since $h(r'^+) = h(r^+)$ and $p(r'^+) = p(r^+)$.
By induction hypothesis, this entails $j \in R'_U$.
Due to $j \notin R_U$, it holds that $i \in Q'_T$, so
there is an atom $a \in I$ such that $\L_{\sigma'} = \L_\sigma(a) = i$.
Due to the new edge from $a$ to $r$,
$a \in n(r)$ holds, which contradicts $n(r) \cap I = \emptyset$.

\medskip\noindent
The other direction is similar.
\end{proof}

Hence we get an FPT result for answer-set semantics:

\begin{THE}
\label{thm:asp-fpt}
Let $k$ be a constant and $\Pi$ be a program.
Given a $k$-expression for the signed incidence graph of $\Pi$, we can decide in linear time whether $\Pi$ has an answer set.
\end{THE}

\begin{proof}
Let $k$ be a constant, $\Pi$ be a program and $\sigma$ be a $k$-expression of $\sincG(\Pi)$.
Using the same ideas as for Theorem~\ref{thm:classical-models-fpt},
we can easily show that there is an answer set of $\Pi$
if and only if there is a $k$-pair $(Q,\Gamma)$ in $g(\sigma)$ such that
$Q_U=\emptyset$ and
$R_U \neq \emptyset$ for every $R \in \Gamma$.
Again, it is easy to see that $g(\sigma)$ can be computed in linear time.
\end{proof}

\subsection{The Role of Signs for Results on Clique-Width}\label{sec:hardscw}

We have shown 
that ASP parameterized by the
clique-width of the signed incidence graph is FPT.
Because the clique-width of the (unsigned) incidence graph
is usually smaller than (and always at most as high as) 
the clique-width of the signed incidence graph (Proposition~\ref{pro:unsigned-signed-cw}), an FPT
result w.r.t.\ the clique-width of the (unsigned) incidence graph would
be a significantly stronger result. It is therefore natural to ask
whether ASP is already FPT 
w.r.t\ the
clique-width of the unsigned incidence graph. A similar situation is known for the
satisfiability problem of propositional formulas (SAT), which was
first shown in~\cite{FischerMakowskyRavve06} to be
FPT parameterized by the clique-width of the
signed incidence graph, and the authors conjectured that the same should
hold already for the unsigned version. Surprisingly, this turned out
not to be the case~\cite{OrdyniakPaulusmaSzeider13}.
In comparison to SAT, the situation for ASP
is similar but slightly more involved. Whereas there are only two
potential signs for SAT (signaling whether a variable occurs
positively or negatively in a clause), ASP has three signs ($h$, $p$,
$n$). So how many signs are necessary to obtain tractability for
ASP? To settle this question, let $\sincG_{L}(\Pi)$, for $L \subseteq
\{h,p,n\}$, be the (``semi-signed'') incidence graph obtained from
$\sincG(\Pi)$ by joining all labels in $L$, i.e., every label
in $L$ is renamed to a new label, which we denote by $\alpha$. We will show below that
joining any set $L$ of labels other than $\{h,n\}$ leads to intractability
for ASP parameterized by the clique-width of
$\sincG_{L}(\Pi)$. Together with our tractability result w.r.t.\ the
clique-width of $\sincG(\Pi)$ (Theorem~\ref{thm:asp-fpt}), this
provides an almost complete picture of the 
complexity of ASP
parameterized by 
clique-width. 
We leave it as an open question whether ASP parameterized by
the clique-width of $\sincG_{\{h,n\}}(\Pi)$ is FPT.
\begin{THE}
  Let $L \subseteq \{h,p,n\}$ with $|L|>1$ and $L\neq \{h,n\}$, then
  ASP is $\W[1]$\hy hard parameterized by the clique-width of
  $\sincG_L(\Pi)$. 
\end{THE}
\begin{proof}
  We show the result by a parameterized reduction from the 
$W[1]$\hy complete problem 
\textsc{Partitioned
    Clique}
~\cite{Pietrzak03}.
  \begin{quote}
   %
    \emph{Instance:} A $k$\hy partite graph $G=(V,E)$ with partition
    $V_1,\dotsc,V_k$ where $|V_i|=|V_j|$ for every $i,j$ with $1
    \leq i,j \leq k$.
    
    \emph{Parameter:} The integer $k$. 
    
    \emph{Question:} Does $G$ have a clique of size $k$?
  \end{quote}
  Recall that a $k$-partite graph is a graph whose vertex set can be
  partitioned into $k$ sets such that there are no edges between
  vertices contained in the same set.
%
	To prove the theorem it is sufficient to show
  that the result holds for $L$ being any combination of two labels other than
  $\{h,n\}$. In other words, it suffices to show the result for
  $L=\{h,p\}$ and $L=\{p,n\}$. Because the reduction for the case that $L=\{h,p\}$ is very similar to the reduction from
  \textsc{Partitioned Clique} to SAT given
  in~\cite[Corollary 1]{OrdyniakPaulusmaSzeider13}, we omit its proof here and only give the proof for the case that $L=\{p,n\}$. 
  Let $L=\{p,n\}$ and $G$, $k$,
  $V_1,\dotsc,V_k$ be as in the definition of \textsc{Partitioned
    Clique} and assume that the vertices of $G$ are labeled such that
  $v_i^j$ is the $i$-th vertex contained in $V_j$. We will construct a program
  $\Pi$ in polynomial-time such that $G$ has a clique of size $k$ if
  and only if $\Pi$ has an answer set, and the clique-width of
  $\sincG_L(\Pi)$ is at most $k'=2k+k^2$. 

  The program
  $\Pi$ contains one atom $v_i^j$ 
  for every vertex $v_i^j$ of $G$, 
 and the
  following rules:

 \noindent
        (R1) For every $j$ with $1\leq j \leq k$, the rule:
    $v_1^j \vee \cdots \vee v_n^j \leftarrow$.

 \noindent
        (R2)
        The rule:
        \begin{align*}
          \leftarrow v_{i_1}^{j_1}, v_{i_2}^{j_2}, & \neg v_1^{j_1}, \dotsc
          ,\neg v_{i_1-1}^{j_1}, \neg v_{i_1+1}^{j_1}, \dotsc , \neg v_{n}^{j_1},\\
          &\neg v_1^{j_2}, \dotsc, \neg v_{i_2-1}^{j_2}, \neg
          v_{i_2+1}^{j_2}, \dotsc ,\neg v_{n}^{j_2}
        \end{align*}
        for every $\{v_{i_1}^{j_1},v_{i_2}^{j_2}\} \notin E(G)$
        with $1\leq j_1,j_2 \leq k$ and $1\leq i_1,i_2 \leq n$.
%
\smallskip

  We first show that $G$ has a clique of size $k$ if and only if
  $\Pi$ has an answer set. Toward showing the forward direction, let
  $C$ be a clique of size $k$ of $G$. We
  claim that $C$ is also an answer set of $\Pi$ and first show
  that $C$ is indeed a model of $\Pi$. 
  Because $C$ contains
  exactly one vertex from every $V_i$, all rules of type (R1) are
  satisfied by $C$. Moreover, the same holds for all rules of type
  (R2), because there is no pair $u,v \in C$ with $\{u,v\} \notin
  E(G)$ and hence the body of every such rule is always
  falsified. This shows that $C$ is a model of $\Pi$. 
  Finally, because all the rules of type (R1) are also
  contained in the reduct $\Pi^C$ of $\Pi$, we obtain that $C$ is
  an answer set of $\Pi$.

  Toward showing the reverse direction, let $C$ be an answer set of
  $\Pi$. We claim that $C$ is also a clique of $G$ of size $k$ and
  first show that $C$ contains exactly one variable from every $V_i$. 
  Because of the
  rules of type (R1) (which will also remain in the reduct $\Pi^C$), it
  holds that $C$ contains at least one variable corresponding to a
  vertex of $V_i$ for every $i$ with $1\leq i \leq k$. 
  Assume for a contradiction that $C$ contains more than one variable from some $V_i$. Then for every $j$ with $j \neq i$, $C$ has to contain at least three variables from $V_i \cup V_j$. Consequently,
  every rule of type (R2) corresponding to a non-edge of $G$ incident
  to a vertex in $V_i$ does not appear in the reduct $\Pi^C$ of $\Pi$,
  which shows that $C$ is not an answer set (minimal model) of
  $\Pi^C$.
  This
  shows that $C$ contains exactly one variable from every
  $V_i$. Now, suppose for a contradiction that $C$ is not a clique of
  $G$ of size $k$. Then there are $u$ and $w$ in $C$ with $u \in V_i$
  and $w \in V_j$ such that $\{u,w\} \notin E(G)$. Hence, there is a
  rule of type (R2) (corresponding to the non-edge $\{u,v\}$), which
  is violated by $C$, a contradiction to our assumption that $C$ is
  model of $\Pi$.

  It remains to show that the clique-width of 
  $\sincG_{L}(\Pi)$ is at most $k'=2k+k^2$. We show this by providing a $k'$-expression for
  $\sincG_{L}(\Pi)$.
  We start by giving the terms that introduce the vertices of
  $\sincG_{L}(\Pi)$:
(1) We introduce every atom vertex $v_i^j$ of
    $\sincG_{L}(\Pi)$ with the term $j(v_i^j)$.
(2) For every rule vertex $r$ of the form (R1) corresponding to a
    rule $v_1^j \vee \cdots \vee v_n^j \leftarrow$, we introduce the
    term $l(r)$ where $l=k+j$.
(3) For every rule vertex $r$ of the form (R2) corresponding to a
    non-edge between $V_i$ and $V_j$ with $1 \leq i < j \leq k$, we
    introduce the term $l(r)$, where $l=2k+k(i-1)+j$.
  We then combine all these terms using the disjoint union operator
  $\oplus$. Now it only remains to show how the edges between the
  rule and the atom vertices of $\sincG_{L}(\Pi)$ are added:
First, 
for every $j$ with $1 \leq j \leq k$, we add the edges between
    the rule vertices of the form (R1) and the atom vertices contained
    in those rules using the operator $\edge{j}{k+j}{\ell}$, where $\ell=h$.
Second, 
for every $i$ and $j$ with $1 \leq i < j \leq k$, we add the edges between
    the rule vertices of the form (R2) and the atom vertices contained
    in those rules using the operators $\edge{i}{2k+k(i-1)+j}{\alpha}$
    and $\edge{j}{2k+k(i-1)+j}{\alpha}$.

\end{proof}

\section{Conclusion}
In this paper, 
we have contributed to the parameterized complexity analysis of ASP.
We first gave some negative observations showing that 
most directed width measures (applied to the dependency
graph or incidence graph of a program) do not lead to FPT results.
On the other hand, we turned a theoretical
tractability result (which implicitly follows from 
previous work~\cite{GottlobPichlerWei10})
for the parameter clique-width (applied to the signed incidence graph of a program)
into a novel dynamic programming algorithm. 
The algorithm is applicable to arbitrary programs, 
whenever a defining $k$-expression is given. 
The algorithm is expected to run efficiently in particular for small~$k$, i.e.,
programs for which the signed incidence graph has low clique-width.

Future work includes 
solving the remaining question 
whether ASP parameterized by the clique-width of
  $\sincG_{\{h,n\}}(\Pi)$ is FPT or not.
Another open question is whether
ASP parameterized by the clique-width of the unsigned incidence graph is in the class XP
(as is the case for SAT~\cite{SlivovskySzeider13}).

\paragraph{Acknowledgments.} This work was supported by the Austrian Science Fund (FWF) projects P25518 and Y698.

\newpage

\sv{\bibliographystyle{ecai}
\bibliography{literature}
}

\lv{

}

\end{document}